\documentclass[letterpaper]{article}
\usepackage{aaai2027}  

\newif\ifarxiv
\arxivfalse
\newcommand{\arxivonly}[1]{\ifarxiv#1\fi}

\usepackage[hyphens]{url}
\usepackage{graphicx}
\usepackage{natbib}              
\usepackage{caption}             
\usepackage{algorithm}
\usepackage{algorithmic}

\usepackage{newfloat}
\usepackage{amsfonts}
\usepackage{amssymb}
\usepackage{amsthm}
\usepackage{thmtools}
\declaretheorem[style=definition]{definition}

\declaretheorem[style=definition]{lemma}
\declaretheorem[style=definition]{proposition}
\declaretheorem[style=definition]{corollary}
\declaretheorem[style=theorem]{theorem}
\declaretheorem[style=remark]{remark}

\usepackage{listings}
\DeclareCaptionStyle{ruled}{labelfont=normalfont,labelsep=colon,strut=off}
\floatstyle{ruled}
\newfloat{listing}{tb}{lst}{}
\floatname{listing}{Listing}

\usepackage{booktabs}
\usepackage{multirow}
\usepackage{amsmath}

\title{SemRD-V2X: Closure-Guided Communication with Bounded Inference for Cooperative Perception}

\author{
    Hu Xu\textsuperscript{\rm 1},
    Chun Li\textsuperscript{\rm 1},
    Siyuan Qiu\textsuperscript{\rm 1},
    Zeyan Li\textsuperscript{\rm 1},
    Jianfeng Xu\textsuperscript{\rm 1}\textsuperscript{\textdagger}
}
\affiliations{
    \textsuperscript{\rm 1}Shanghai Jiao Tong University, Shanghai, China\\
    \texttt{xuhu6736@sjtu.edu.cn}, \texttt{xujf@sjtu.edu.cn}\\
}

\begin{document}

\maketitle

\begin{abstract}
Vehicle-to-Everything (V2X) cooperative perception improves 3-D detection by sharing intermediate features, but dense remote features may repeat context that the ego agent can infer locally. Most communication-efficient designs optimize masks or codes empirically, leaving a more basic question open: which remote evidence is indispensable given the receiver's own observation? We introduce a closure-fidelity perspective on ego conditioned remote perception. Under a finite deductive abstraction and explicit conditions, its rate--distortion function decomposes over an irredundant core, and the exact zero-distortion rate becomes $P_A H(\pi_A)$. This analysis suggests a concrete design principle: transmit compact evidence and recover derivable context with bounded receiver-side inference. Guided by this principle, SemRD-V2X is an operational neural proxy that combines exact-budget BEV support selection, pointwise channel compression, and masked shared-weight reconstruction before standard fusion. Experiments on simulated V2XSet and real-world DAIR-V2X
validate the resulting design. In a controlled five-run V2XSet comparison against a locally reproduced V2X-ViT-v1 baseline on one Tesla V100, SemRD-V2X reduces the analytical feature payload by $26.6\times$ while improving AP@0.5/AP@0.7 by 4.13/8.57 points, with 3.81\% additional mean compute latency. These results position closure fidelity as both an analytical lens and an actionable design principle for communication-efficient cooperative perception.
\end{abstract}







\section{Introduction}
\label{sec:intro}

Vehicle-to-Everything (V2X) cooperative perception extends an ego vehicle's
perception range and mitigates occlusion by combining observations from nearby
vehicles and roadside infrastructure. Intermediate-feature fusion balances
the rich evidence of raw-sensor sharing against the lower cost of exchanging
final detections \citep{v2xvit2022,xu2024v2x,where2comm2022}. Public
benchmarks span both simulated vehicle-to-vehicle settings and real-world
vehicle--infrastructure deployments \citep{opv2v2021,dairv2x2022}. However, a dense
remote BEV tensor may contain context that is predictable from selected remote
evidence together with the ego observation. This raises a question more
fundamental than how to compress the tensor:
\emph{which remote perceptual content must remain explicit, and which part can
be replaced by bounded receiver-side inference?}

Communication-efficient methods reduce traffic through spatial masks,
utility-aware policies, compact bottlenecks, quantization, or codebooks
\citep{who2com2020,when2com2020,where2comm2022,how2comm2023,cobevt2022,
codefilling2024}. These methods
provide effective empirical rate--accuracy trade-offs, but they typically
treat transmission as task-driven sparsification or coding rather than
characterizing which evidence is deductively indispensable given receiver
side information. Classical and task-oriented rate--distortion theory provides
a general optimization framework
\citep{shannon1948,cover2006,tishby1999,bourtsoulatze2019}, but its conclusions
depend on the distortion measure. Common symbol-, feature-, and task-level
criteria do not equate messages by their inferential consequences. What is
missing is a fidelity criterion that separates evidence that must be
communicated from context that remains derivable at the receiver.

We introduce a \emph{closure-fidelity} perspective~\citep{xu2026rate} for
ego-conditioned cooperative perception. Under a finite deductive abstraction,
perceptual features are treated as statements linked by a shared inference
system, and two representations are equivalent when they induce the same
deductive closure. This separates an irredundant core from derivable
statements. When reconstructions are restricted to the source
closure, the rate--distortion function decomposes as
$R_{\mathrm{V2X}}(D)=P_A R^{(A)}(D/P_A)$, where $P_A$ is the core mass and
$R^{(A)}$ is the rate--distortion function of the core sub-source. When
distinct core statements have disjoint zero-distortion reconstruction sets,
$R_{\mathrm{V2X}}(0)=P_A H(\pi_A)$. Bounding the decoder to at most $\delta$
inference steps yields a zero-distortion rate that is non-increasing in
$\delta$ under explicit disjointness and order-robustness conditions.

We translate this perspective into \emph{SemRD-V2X}, the operational
communication design shown in Figure~\ref{fig:semrd-overview}. A lightweight
scorer learns remote BEV importance through a continuous Gumbel relaxation
during training and selects an exact-budget Top-$k$ support at inference. A
pointwise codec reduces the channel cost of selected positions, and a
shared-weight Differentiable Inference Module (DIM) reconstructs only omitted
positions for $\delta$ steps before V2X-ViT fusion. Thus $\rho$ controls the
analytical feature payload, while $\delta$ controls receiver-side reconstruction
effort. These mechanisms operationalize structural principles suggested by
closure fidelity.

Across V2XSet and DAIR-V2X, SemRD-V2X remains competitive in simulated and
real-world evaluation. In a controlled five-run comparison on one Tesla V100,
it reduces the analytical feature payload of a locally reproduced V2X-ViT-v1
baseline by $26.6\times$ and improves AP@0.5/AP@0.7 by $4.13/8.57$ points,
with $3.81\%$ higher mean compute latency. These results support a favorable
empirical communication--accuracy--computation trade-off. Our contributions
are threefold:
\begin{itemize}
    \item We introduce a closure-fidelity perspective for ego-conditioned V2X
    perception and derive conditional core-only and bounded-depth rate
    characterizations under a finite deductive abstraction.
    \item We translate the resulting principles into SemRD-V2X, an operational
    design with exact-budget support selection, pointwise coding, and bounded
    shared-weight reconstruction.
    \item We evaluate its communication--accuracy--computation trade-off on
    V2XSet and DAIR-V2X through rate and depth controls, robustness tests,
    component ablations, and a controlled payload--latency comparison.
\end{itemize}

\section{Related Work}
\label{sec:related}

\begin{figure*}[ht]
    \centering
    \includegraphics[width=\textwidth]{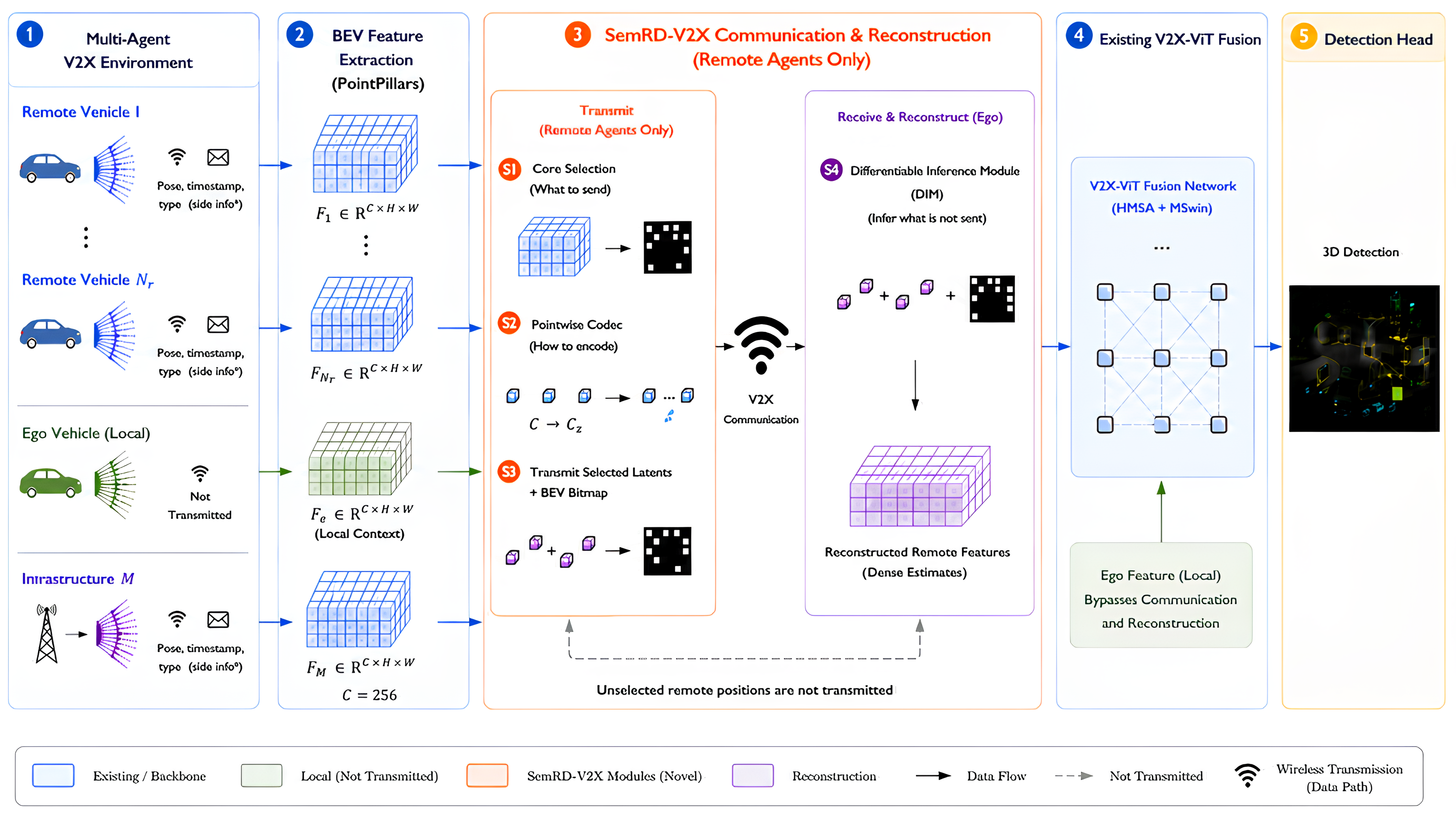}
    \caption{SemRD-V2X encodes selected remote BEV positions and a spatial
    bitmap. The receiver reconstructs omitted positions through bounded
    refinement before V2X-ViT fusion. Analytical payload includes latents and
    the bitmap but excludes the ego feature and network overhead.}
    \label{fig:semrd-overview}
\end{figure*}

\textbf{Cooperative Perception and Communication Efficiency.}
Intermediate-fusion methods aggregate learned BEV features through feature
fusion, message passing, distillation, transformers, or heterogeneous feature
alignment, while system-level work frames communication as an edge-assisted
autonomous-driving constraint
\citep{fcooper2019,opv2v2021,v2vnet2020,disconet2021,v2xvit2022,
xu2024v2x,heal2024,edgev2x2021}. Communication-efficient variants instead
learn when, where, or at what precision to communicate. Who2com and When2com learn which
collaborators to query or when communication should occur through handshake
and graph-grouping mechanisms \citep{who2com2020,when2com2020}.
Where2Comm selects informative BEV locations
\citep{where2comm2022}; How2Comm couples utility-aware selection with compact
messages \citep{how2comm2023}; other systems use task-dependent suppression,
sparse BEV processing, or codebook representations
\citep{chiu2023selective,cobevt2022,codefilling2024}. Latency compensation and
pose-robust fusion address complementary temporal and spatial misalignment
effects \citep{latencyaware2022,robcoop2023}. These methods establish spatial
selection and compact coding as effective primitives, but their objectives
generally do not specify which remote evidence remains indispensable once the
ego feature is available or define deductive fidelity for omitted context.
SemRD-V2X instead uses closure fidelity to motivate their joint operation with
bounded shared-weight reconstruction; selection and coding are not claimed as
novel in isolation.

\textbf{Rate--Distortion Coding with Side Information.}
Classical rate--distortion theory minimizes statistical rate under a prescribed
distortion \citep{shannon1948,cover2006}, while Slepian--Wolf and Wyner--Ziv
coding exploit correlated sources or decoder side information
\citep{slepian1973noiseless,wyner1976rate}. The information bottleneck,
DeepJSCC, semantic communication, and rate--distortion--perception theory
emphasize downstream relevance, end-to-end transmission, preservation of
meaning, or perceptual quality
\citep{tishby1999,bourtsoulatze2019,xie2021deepsc,blau2019}. These frameworks exploit
statistical dependence, receiver side information, or task relevance, but
they do not define fidelity as equality of deductive consequences. Our setting
asks instead whether the transmitted remote evidence and the ego observation
preserve the same deductive consequences as the complete remote message.
For finite alphabets, the Blahut and Arimoto procedures provide standard
computational routes for evaluating classical information-theoretic quantities
\citep{blahut1972,arimoto1972}; our analysis changes the fidelity criterion
rather than introducing another numerical solver for the classical objective.

\textbf{Deductive Coding and V2X Positioning.}
Deductive source coding defines fidelity through equality of deductive closures
and studies zero-distortion coding, bounded inference, and restricted
reconstruction vocabularies for abstract finite sources \citep{xu2026rate}.
We specialize this framework to ego-conditioned V2X by treating the ego feature
as decoder context and remote observations as the source, yielding core-only
and depth-constrained rate characterizations. SemRD-V2X operationalizes these
implications through budgeted support selection, pointwise coding, and bounded
local reconstruction.


\section{Methodology}
\label{sec:methodology}

SemRD-V2X communicates remote BEV features under a fixed spatial budget and
reconstructs omitted context before ego-conditioned fusion. Guided by the
closure-fidelity abstraction below, it combines inference-time exact-budget
support selection, pointwise channel coding, and shared-weight masked
reconstruction: $\rho$ controls transmission, while $\delta$ controls
receiver-side refinement at fixed payload.

\subsection{Ego-Conditioned Closure-Fidelity Model}

Let $\mathcal E$ be a finite set of ego statements available at the receiver
and $\mathcal S$ a finite set of remote statements to be communicated. A shared
proof system induces the contextual closure
\begin{equation}
\operatorname{Cn}_{\mathcal E}(\mathcal T)
=\{s:\mathcal E\cup\mathcal T\vdash s\},
\qquad \mathcal T\subseteq\mathcal S.
\label{eq:contextual-closure}
\end{equation}
A fixed canonical deletion order yields an irredundant generator
$A=\operatorname{Atom}(\mathcal S)$ such that
$\operatorname{Cn}_{\mathcal E}(A)=\operatorname{Cn}_{\mathcal E}(\mathcal S)$
and no $a\in A$ is derivable from $A\setminus\{a\}$. For a source statement
$s$ and reconstruction $\hat s$, let
$\mathcal S_{s\rightarrow\hat s}=(\mathcal S\setminus\{s\})\cup\{\hat s\}$
and write $\mathcal C=\operatorname{Cn}_{\mathcal E}(\mathcal S)$ and
$\widehat{\mathcal C}=\operatorname{Cn}_{\mathcal E}
(\mathcal S_{s\rightarrow\hat s})$. Closure distortion is
\begin{equation}
d_{\mathrm{V2X}}(s,\hat s\mid\mathcal S)
=1-\frac{|\mathcal C\cap\widehat{\mathcal C}|}
{|\mathcal C\cup\widehat{\mathcal C}|}.
\label{eq:closure-distortion}
\end{equation}
Thus zero distortion means preservation of deductive consequences, not equality
of feature vectors. For $S\sim P_S$ and reconstruction $\widehat S$, define
$R_{\mathrm{V2X}}(D)$ by minimizing $I(S;\widehat S)$ subject to expected
closure distortion at most $D$.

\begin{theorem}[Core-only rate in the deductive abstraction]
\label{thm:core-rate}
Let $P_A=P_S(A)$ and, when $P_A>0$, $\pi_A=P_S(\cdot\mid A)$. If the
reconstruction alphabet is contained in
$\operatorname{Cn}_{\mathcal E}(\mathcal S)$, then
\begin{equation}
R_{\mathrm{V2X}}(D)
=P_A R^{(A)}(D/P_A).
\label{eq:core-rate}
\end{equation}
Here $R^{(A)}$ is the rate--distortion function of $S$ conditioned on $S\in A$;
the right-hand side is zero when $P_A=0$. If, additionally, $P_A>0$, $A$ is
representable at the receiver, and
distinct core statements have disjoint zero-distortion reconstruction sets,
then $R_{\mathrm{V2X}}(0)=P_A H(\pi_A)$.
\end{theorem}
\begin{proof}[Proof sketch]
Statements in $\mathcal S\setminus A$ are closure-redundant, so only the core
mass $P_A$ contributes distortion. Restricting a test channel to $A$ gives the
converse, while extending an optimal core channel with its output marginal
gives achievability. At $D=0$, disjoint reconstruction sets identify the core
symbols, yielding $P_A H(\pi_A)$~\citep{xu2026rate}.
\end{proof}

If closure is truncated to at most $\delta$ applications of a monotone
finite-step operator, let $A_\delta$ contain the statements not derivable from
the remainder within $\delta$ steps and $P_\delta=P_S(A_\delta)$. Under
analogous representability and disjointness conditions, the same argument gives
\begin{equation}
R_{\mathrm{V2X}}(D,\delta)
=P_\delta R^{(A_\delta)}(D/P_\delta).
\label{eq:depth-rate}
\end{equation}
Under the stated disjointness and order-robustness conditions,
$R_{\mathrm{V2X}}(0,\delta)$ is non-increasing in the allowed depth. A
separate theory-only result covers restricted receiver vocabularies under
closure-equivalent representability and disjointness; SemRD-V2X neither
instantiates nor tests these conditions.

\subsection{From Structural Principles to an Operational Design}

The formal model does not prescribe a neural architecture. Instead, its
structural implications motivate three implementation principles: communicate
only a budgeted remote support, anchor the decoded evidence during omitted-
context reconstruction, and bound reconstruction independently of the
payload. Let
$\mathcal A=\{e\}\cup\mathcal A_r$ contain the ego and remote agents. A shared
PointPillars-style backbone~\citep{pointpillars2019} produces
$\mathbf F_i\in\mathbb R^{C\times H\times W}$, but only remote features cross
the link. For remote agent $i$, the resulting inference-time hard-mask
implementation satisfies
\begin{gather}
|\widehat A_i|=k,\qquad
\mathbf M_i^\rho\odot\mathbf X_i^{(t)}
=\mathbf M_i^\rho\odot\mathbf X_i^{(0)}
\quad \forall t\in\{0,\ldots,\delta\},\nonumber\\
\mathbf X_i^{(\delta)}
=\mathbf X_i^{(0)}+(1-\mathbf M_i^\rho)\odot
\mathbf R_i^{(\delta)},
\label{eq:method_constraints}
\end{gather}
where $k=\max\{1,\lfloor\rho HW\rfloor\}$, $\widehat A_i$ is the transmitted
support, $\mathbf M_i^\rho\in\{0,1\}^{H\times W}$ is its Top-$k$ mask, and
$\mathbf R_i^{(\delta)}$ is the cumulative omitted-region refinement. These
invariants form the theory-to-design bridge: communicate exactly $k$ spatial
positions, preserve their decoded evidence throughout refinement, and update
only omitted positions for at most $\delta$ steps. Thus $\rho$ controls the
inference-time spatial payload, while $\delta$ controls reconstruction
computation without changing that payload.

SemRD-V2X is therefore an operational neural proxy guided by closure fidelity,
not a solver for the abstract rate--distortion problem. It neither recovers a
symbolic core nor equates its analytical payload or neural diagnostics with the
abstract rate or closure distortion. The scorer and DIM receive only the remote
tensor and mask, while $\mathbf F_e$ enters at fusion; hence the overall decoder
is ego-conditioned, but inference-time support is not conditioned on the
current ego feature. Learned-versus-random support and matched-payload
reconstruction ablations consequently provide operational tests of the design,
not symbolic verification.

\subsection{Training-Relaxed Selection with an Exact Inference Budget}

To allocate a fixed spatial budget content-adaptively, a lightweight scorer
ranks the BEV positions of each remote agent $i\in\mathcal A_r$. At $(h,w)$,
it concatenates $\mathbf F_i(:,h,w)$ with five metadata values: normalized
coordinates, log $3\times3$ local feature variance, normalized maximum channel
magnitude, and normalized distance to the BEV center. A two-layer pointwise
MLP with 128 hidden units and ReLU activation produces
\begin{equation}
\alpha_i(h,w)=f_{\mathrm{score}}(
[\mathbf F_i(:,h,w),\mathbf m_i(h,w)]),
\qquad
\alpha_i\in\mathbb R^{H\times W}.
\label{eq:method_score}
\end{equation}
Training uses a differentiable Gumbel-threshold relaxation. Let
$\mathbf g_i$ contain i.i.d. standard Gumbel samples and define
\begin{align}
\boldsymbol\beta_i&=\alpha_i+\mathbf g_i,\qquad
\kappa_i^{\mathrm{tr}}=\operatorname{kth}(\boldsymbol\beta_i,k),\\
\mathbf S_i(h,w)
&=\sigma\!\left(
\frac{\boldsymbol\beta_i(h,w)-\kappa_i^{\mathrm{tr}}}{\tau}
\right).
\label{eq:method_soft_mask}
\end{align}
Here $\operatorname{kth}$ returns the $k$-th largest noisy score. The
continuous mask $\mathbf S_i$ is used in the training forward pass, and
the temperature $\tau$ is annealed to sharpen the relaxation.

At inference, no Gumbel perturbation or soft mask is used. The selected set and
binary mask are
\begin{equation}
\widehat A_i(\rho)=\operatorname{TopK}(\alpha_i,k),
\qquad
\mathbf M_i^\rho(h,w)=\mathbb I[(h,w)\in\widehat A_i(\rho)].
\label{eq:method_hard_mask}
\end{equation}
Top-$k$ index selection gives $|\widehat A_i|=k$ exactly. Consequently, the
strict cardinality and payload claims apply to inference; the training mask is
a continuous optimization surrogate and need not have mean exactly $\rho$.
For compact notation, let
\begin{equation}
\widetilde{\mathbf M}_i=
\begin{cases}
\mathbf S_i, & \text{during training},\\
\mathbf M_i^\rho, & \text{during inference}.
\end{cases}
\label{eq:method_active_mask}
\end{equation}

\subsection{Pointwise Coding and Support-Anchored Bounded Reconstruction}

To reduce per-position cost without spatially mixing omitted context, the
codec compresses only along the channel dimension:
\begin{equation}
\mathbf Z_i=Q_b\!\left(g_{\mathrm{enc}}(\mathbf F_i)\right),
\qquad
\mathbf X_i^{(0)}=
g_{\mathrm{dec}}\!\left(
\widetilde{\mathbf M}_i\odot\mathbf Z_i
\right),
\label{eq:method_codec}
\end{equation}
where $g_{\mathrm{enc}}$ and $g_{\mathrm{dec}}$ are bias-free $1\times1$
projections, $\mathbf Z_i\in\mathbb R^{C_z\times H\times W}$, and $Q_b$
simulates the configured transport precision by a cast round trip. The
evaluated configuration uses $C_z=64$ and FP16 ($b=16$); the analytical
payload assumes no entropy coder. Because the codec is pointwise, it changes
the channel representation without spatially mixing selected and omitted
positions before reconstruction.

To reconstruct missing context without further communication, the receiver
performs $\delta$ masked refinement steps. Starting from the decoded feature
$\mathbf X_i^{(0)}$, one iteration computes
\begin{align}
\mathbf P_i^{(t)}
&=W_p\,\mathrm{GELU}\!\left(
\mathrm{GN}(W_d*\mathbf X_i^{(t)})
\right),\\
\mathbf G_i^{(t)}
&=\sigma\!\left(
W_g*[\operatorname{mean}_c|\mathbf X_i^{(t)}|,
\widetilde{\mathbf M}_i]
\right),\\
\mathbf X_i^{(t+1)}
&=\mathbf X_i^{(0)}
+(1-\widetilde{\mathbf M}_i)\odot
\left(
\mathbf X_i^{(t)}
+\mathbf G_i^{(t)}\odot\mathbf P_i^{(t)}
\right).
\label{eq:method_dim}
\end{align}
Here $W_d$ is a depthwise $3\times3$ convolution, GN uses eight groups, and
$W_p$ is a pointwise channel mixer. The $3\times3$ convolution $W_g$ predicts
a scalar spatial gate from feature energy and the active mask. During
inference, the hard mask re-anchors every transmitted position after each
iteration; because $\mathbf X_i^{(0)}$ is zero at
unselected hard-mask positions, refinement changes only omitted positions.
During training, the same implemented update uses the continuous mask in
Eq.~\eqref{eq:method_active_mask}. All iterations share parameters; changing
$\delta$ changes the effective receptive field and computation without changing
the number of learned parameters.

After $\delta$ iterations, the reconstructed remote features
$\{\mathbf X_i^{(\delta)}\}_{i\in\mathcal A_r}$ are fused with the complete ego
feature $\mathbf F_e$ by the standard attention-based V2X-ViT detector
\citep{vaswani2017,v2xvit2022}.

\subsection{Optimization and Operational Diagnostics}

The optimized objective matches the checkpoint used in all reported full-model
experiments and jointly trains the selector, codec, DIM, and detector:
\begin{equation}
\mathcal L
=\mathcal L_{\mathrm{det}}
+\lambda_{\mathrm{rate}}\mathcal L_{\mathrm{rate}}
+\lambda_{\mathrm{rec}}\mathcal L_{\mathrm{rec}}.
\label{eq:method_loss}
\end{equation}
No detached full-message teacher or closure-consistency loss is used.

\paragraph{Score-concentration regularization.}
Let
\begin{equation}
\mathbf q_i=\operatorname{softmax}(\alpha_i/\tau),
\qquad
\mathcal L_{\mathrm{rate}}
=\frac{\rho}{N_r}\sum_{i\in\mathcal A_r}
\frac{H(\mathbf q_i)}{\log(HW)},
\label{eq:method_rate}
\end{equation}
where $N_r=|\mathcal A_r|$. Minimizing this normalized entropy encourages a
concentrated score distribution, where
$H(\mathbf q_i)=-\sum_{h,w}q_i(h,w)\log q_i(h,w)$. The factor $\rho$ is the
configured target fraction; because inference-time Top-$k$ already fixes the
transmitted cardinality, this regularizer changes the learned ranking rather than the
analytical payload. Its entropy is not an estimate of the theoretical
$H(\pi_A)$.

\paragraph{Omitted-region reconstruction supervision.}
During training, the DIM is supervised against the detached pre-mask remote
feature only where the relaxed mask omits content. Let $\operatorname{sg}$
denote stop-gradient and define
$\mathbf W_i=\operatorname{sg}(1-\mathbf S_i)$. The implemented loss is
\begin{equation}
\begin{aligned}
\mathcal L_{\mathrm{rec}}
&=\frac{1}{Z_{\mathrm{rec}}}
\sum_{i\in\mathcal A_r}\sum_{c,h,w}\mathbf W_i(h,w)\\
&\quad\cdot\ell_{\mathrm{SL1}}\!\left(
\mathbf X_i^{(\delta)}(c,h,w),
\operatorname{sg}(\mathbf F_i(c,h,w))\right),\\
Z_{\mathrm{rec}}
&=\max\!\left\{
C\sum_{i\in\mathcal A_r}\sum_{h,w}\mathbf W_i(h,w),1
\right\}.
\end{aligned}
\label{eq:method_rec}
\end{equation}
where $\ell_{\mathrm{SL1}}$ is elementwise Smooth-$L_1$. Detaching the weights
prevents this term from changing the selector merely by altering the measured
omitted area; it trains the codec and DIM to recover missing features under the
current relaxed mask.

\paragraph{Post-hoc payload and operational diagnostics.}
At inference, the analytical transmitted bytes for one cooperative frame are
\begin{equation}
B_{\mathrm{frame}}
=\sum_{i\in\mathcal A_r}
\left(
|\widehat A_i|\frac{C_zb}{8}
+\mathbb I[\rho<1]\frac{HWb_{\mathrm{mask}}}{8}
\right),
\label{eq:method_payload}
\end{equation}
where $|\widehat A_i|=k$ and $b_{\mathrm{mask}}=1$. The bitmap identifies
selected positions on the shared BEV lattice, so no coordinate list is counted.
Post-hoc diagnostics comprise channel-summed $L_1$ errors on selected and
omitted positions ($E_{\mathrm{sup}}$ and $E_{\mathrm{omit}}$), sparse--full
post-NMS agreement (PCF), and selected--saliency Top-$K$ overlap (SA@K); their
experimental roles are analyzed in Section~\ref{sec:exp-ablation}.

\section{Experiments}
\label{sec:exp}

We ask four questions: whether SemRD-V2X is competitive on simulated and
real-world benchmarks; whether learned support and reconstruction supervision
help at matched payload; whether $\rho$ and $\delta$ expose separate
communication-budget and bounded-reconstruction controls; and what robustness
and latency boundaries emerge under localization noise and timed inference.

\subsection{Experimental Protocol}
\label{sec:exp-setup}

\begin{table*}[t]
\centering
\small
\setlength{\tabcolsep}{4.0pt}
\begin{tabular}{@{}lcccccc@{}}
\toprule
& \multicolumn{4}{c}{V2XSet} & \multicolumn{2}{c}{DAIR-V2X} \\
\cmidrule(lr){2-5}\cmidrule(lr){6-7}
& \multicolumn{2}{c}{Perfect} & \multicolumn{2}{c}{Noisy} &
\multicolumn{2}{c}{Mean(P,N)} \\
\cmidrule(lr){2-3}\cmidrule(lr){4-5}\cmidrule(lr){6-7}
Method & AP@0.5$\uparrow$ & AP@0.7$\uparrow$ & AP@0.5$\uparrow$ &
AP@0.7$\uparrow$ & AP@0.5$\uparrow$ & AP@0.7$\uparrow$ \\
\midrule
Late Fusion & 0.727 & 0.620 & 0.549 & 0.307 & 0.520 & 0.422 \\
Early Fusion& 0.819 & 0.710 & 0.720 & 0.384 & 0.525 & 0.366 \\
F-Cooper    & 0.840 & 0.680 & 0.715 & 0.469 & 0.550 & 0.415 \\
OPV2V       & 0.807 & 0.664 & 0.709 & 0.487 & 0.548 & 0.420 \\
V2VNet      & 0.845 & 0.677 & 0.791 & 0.493 & 0.523 & 0.427 \\
Where2Comm  & 0.855 & 0.654 & 0.820 & 0.534 & 0.637 & 0.489 \\
How2Comm    & 0.867 & 0.720 & 0.841 & 0.671 & 0.623 & 0.471 \\
V2X-ViT-v1 & 0.882 & 0.712 & 0.836 & 0.614 & 0.561 & 0.446 \\
V2X-ViTv2  & 0.887 & 0.750 & \textbf{0.846} & \textbf{0.672} & 0.634 & 0.541 \\
\midrule
SemRD-V2X & \textbf{0.898} & \textbf{0.789} & 0.845 & 0.647 &
\textbf{0.676} & \textbf{0.553} \\
\bottomrule
\end{tabular}
\caption{3-D detection on V2XSet and DAIR-V2X. V2XSet reports AP separately
under Perfect (P) and Noisy (N), whereas DAIR-V2X reports Mean(P,N).}
\label{tab:main-comparison}
\end{table*}

\begin{table*}[htbp]
\centering
\small
\setlength{\tabcolsep}{4.0pt}
\begin{tabular}{@{}lrrrrrr@{}}
\toprule
Method & AP@0.5$\uparrow$ & AP@0.7$\uparrow$ & MiB/frame$\downarrow$ &
\shortstack{Measured compute\\(mean $\pm$ SD)$\downarrow$} &
\shortstack{Analytical comm.\\(150 Mbit/s)$\downarrow$} &
\shortstack{Illustrative total\\(mean $\pm$ SD)$\downarrow$} \\
\midrule
V2X-ViT-v1 & 0.8033 & 0.5606 & 13.6416 &
\textbf{144.348 $\pm$ 0.465} & 762.895 &
907.243 $\pm$ 0.465 \\
SemRD-V2X (Full) & \textbf{0.8445} & \textbf{0.6463} &
\textbf{0.5131} & 149.845 $\pm$ 3.589 &
\textbf{28.697} & \textbf{178.542 $\pm$ 3.589} \\
\bottomrule
\end{tabular}
\caption{Controlled V100 comparison under the V2XSet Noisy protocol using
epoch-30 checkpoints. All latencies are in ms/frame. Measured compute and the
illustrative sequential total report mean $\pm$ sample SD across five paired
runs; communication is analytical serialization at 150 Mbit/s. AP is
remeasured on the V100.}
\label{tab:latency}
\end{table*}

\begin{figure*}[!h]
    \centering
    \includegraphics[width=0.98\textwidth]{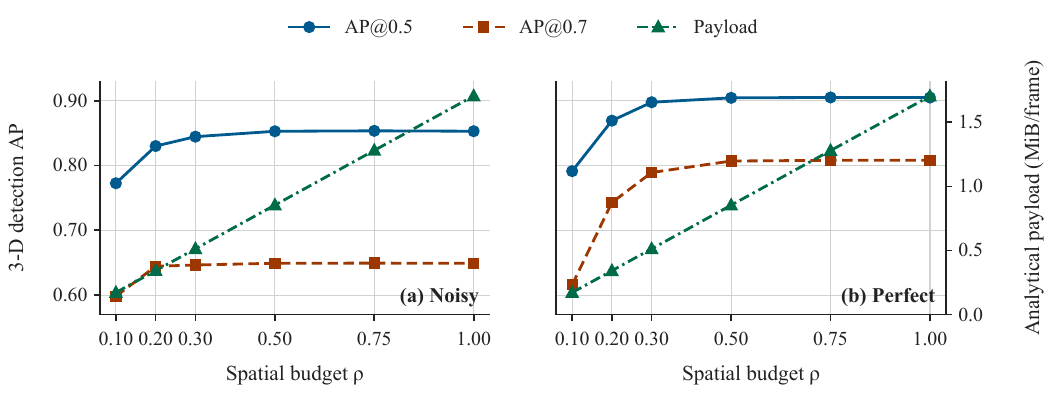}
    \caption{Spatial-budget sweeps at $\delta=2$. Axes report $\rho$, detection
    AP, and analytical payload (MiB/frame). (a) Noisy. (b) Perfect.}
    \label{fig:compression-sweeps}
\end{figure*}

\paragraph{Datasets and task.}
We evaluate multi-agent 3-D vehicle detection on the simulated V2XSet and
real-world vehicle--infrastructure DAIR-V2X benchmarks
\citep{xu2024v2x,dairv2x2022}. Following V2X-ViTv2, the nominal planar ranges
are $x\in[-140,140]$ m, $y\in[-40,40]$ m for V2XSet and
$x\in[-100,100]$ m, $y\in[-40,40]$ m for DAIR-V2X. The public V2X-ViT-v1
configuration used for our controlled V2XSet experiments realizes the former
with voxel-aligned LiDAR bounds $[-140.8,-38.4,-3,140.8,38.4,1]$ m.

\paragraph{Model training and operating point.}
The following controlled settings apply to V2XSet. We train the full model for
30 epochs using Adam with batch size 2, initial
learning rate $10^{-3}$, and weight decay $10^{-4}$. The primary operating
point is $\rho=0.3$, $\delta=2$, $C=256$, and a 64-channel FP16 transmission
bottleneck. The Gumbel temperature decreases linearly from $5.0$ to $0.5$ over
training. All full-model evaluations use the same epoch-30 checkpoint; sweeps
vary only the stated inference control or evaluation condition.

\paragraph{Protocols.}
We keep the two standard evaluation protocols separate throughout the paper. \emph{Perfect} assumes synchronized observations and accurate poses.
\emph{Noisy} uses the standard asynchronous protocol with a fixed 100-ms
temporal offset and Gaussian transmitter-pose perturbation. Unless otherwise
stated, the primary Noisy protocol uses $\sigma_{xyz}=0.2$ m and
$\sigma_{\mathrm{yaw}}=0.2^\circ$. For robustness, we also evaluate a
high-noise setting with $\sigma_{xyz}=0.5$ m, the same
$\sigma_{\mathrm{yaw}}=0.2^\circ$, and the same temporal offset.
DAIR-V2X follows its standard benchmark setting and assesses real-world
detection.
\arxivonly{Code-level units and seed assignments are documented in
Appendix~\ref{app:eval-implementation}.}

\paragraph{Metrics.}
We report AP@0.5 and AP@0.7 for 3-D vehicle detection. Communication is reported
as the \emph{analytical payload} in MiB/frame computed from
Eq.~\eqref{eq:method_payload}; it includes the selected latent values and
one-bit spatial bitmap, while excluding
packet headers, entropy coding, channel coding, retransmissions, and wireless
contention. The ablation table additionally reports the post-hoc diagnostics
defined in Section~\ref{sec:methodology}: support and omission errors, PCF, and
SA@K. Benchmark AP assesses accuracy competitiveness; payload and latency are
compared only against the locally reproduced V2X-ViT-v1 under the matched
protocol below.

\paragraph{Timed inference and statistics.}
We run five paired repetitions on one Tesla V100-PCIE-32GB with batch size one
under the Noisy protocol. Each timed frame is CUDA-synchronized; timing includes
host-to-device transfer, model forward, and post-processing but excludes data
loading. We report the mean and sample standard deviation of repetition-level
mean frame times, reusing one fixed checkpoint and evaluation set per method.

\subsection{Benchmark Accuracy and Controlled Efficiency}
\label{sec:exp-main-comparison}

\begin{table*}[htbp]
\centering
\small
\setlength{\tabcolsep}{4.0pt}
\begin{tabular}{@{}lccccccc@{}}
\toprule
Variant & MiB/frame$\downarrow$ & AP@0.5$\uparrow$ & AP@0.7$\uparrow$ &
PCF@0.5$\uparrow$ & SA@K$\uparrow$ & $E_{\mathrm{sup}}\downarrow$ & $E_{\mathrm{omit}}\downarrow$ \\
\midrule
Full                         & 0.5131 & \textbf{0.8445} & \textbf{0.6465} & \textbf{0.9357} & \textbf{0.4917} & 18.1582 & \textbf{7.2013} \\
Random selector              & 0.5131 & 0.6528 & 0.5032 & 0.9178 & 0.2998  & 8.9167 & 11.1486 \\
w/o rate regularization      & 0.5131 & 0.8148 & 0.5974 & 0.9301  & 0.4650 & 17.4119 & 9.5521 \\
w/o reconstruction supervision & 0.5131 & 0.8272 & 0.5951 & 0.9115  & 0.4710 & 21.9847 & 9.3176 \\
w/o codec                    & 4.0935 & 0.7976 & 0.5996 & 0.9254 & 0.4714 & 0.0000 & 11.5870 \\
\bottomrule
\end{tabular}
\caption{Matched component ablations under the V2XSet Noisy protocol. Each
variant is trained for 30 epochs and evaluated at $\rho=0.3$ and $\delta=2$.
Payload is analytical; PCF, SA@K, $E_{\mathrm{sup}}$, and
$E_{\mathrm{omit}}$ are post-hoc diagnostics.}
\label{tab:ablation}
\end{table*}

Table~\ref{tab:main-comparison} compares methods reported on both benchmarks
\citep{where2comm2022,v2xvit2022,how2comm2023,xu2024v2x}. On
V2XSet, SemRD-V2X obtains AP@0.5/AP@0.7 of 0.8445/0.6465 under Noisy and
0.8976/0.7892 under Perfect, the highest values in the latter comparison. On
DAIR-V2X, it reaches 0.676/0.553, exceeding the strongest compared AP@0.5/AP@0.7
by 3.9/1.2 points and demonstrating applicability to a real-world
vehicle--infrastructure benchmark. These accuracy comparisons establish
competitiveness, but communication and latency claims require the matched
local comparison below.

\paragraph{Matched payload and latency.}
\label{sec:exp-latency}
Table~\ref{tab:latency} re-evaluates both epoch-30 checkpoints under the same
V100 protocol, whereas Table~\ref{tab:main-comparison} reports benchmark AP;
only the former supports matched payload and latency claims. The 100-ms offset
selects an earlier remote timestamp without adding a wall-clock wait. At the
illustrative 150-Mbit/s goodput motivated by an experimental 5G-NR-V2X
study~\citep{10625456}, analytical serialization is added sequentially to
measured computation. SemRD-V2X reduces payload by 96.2\% (26.6$\times$), while
mean compute increases by $5.497\pm3.844$ ms/frame (3.81\%); we therefore do not
claim faster neural inference. The resulting illustrative total decreases from
907.243 to 178.542 ms/frame (80.3\%), and the communication saving exceeds the
compute overhead below the 20.0-Gbit/s break-even goodput. These are analytical
scenarios, not measured end-to-end wireless latency.
\arxivonly{Additional implementation details are provided in
Appendix~\ref{app:eval-implementation}.}
\arxivonly{Appendix~\ref{app:bandwidth-scenarios} details the calculation and
illustrative link-rate scenarios.}

\subsection{Operational Validation of the Design}
\label{sec:exp-ablation}

Table~\ref{tab:ablation} tests the selector, rate regularization,
reconstruction supervision, and codec under the matched Noisy setting. At the
same 0.5131-MiB/frame payload, random selection loses 19.17/14.33
AP@0.5/AP@0.7 points and substantially reduces SA@K, supporting
content-adaptive support selection. Removing rate regularization loses
2.97/4.91 points, while removing reconstruction supervision loses 1.73/5.14
points, with the larger AP@0.7 drop indicating a stronger effect on precise
localization. Removing the codec increases payload by $7.98\times$ while
lowering both AP metrics by 4.69 points. These matched comparisons support
learned selection, compact coding, and reconstruction supervision at the
chosen operating point.

$E_{\mathrm{sup}}$ is conditioned on selected positions, while PCF measures
sparse--full agreement rather than ground-truth correctness. The post-hoc
metrics therefore complement payload and AP but should not be interpreted
independently.

\subsection{Communication--Computation Controls and Robustness}

\paragraph{Spatial-budget control.}
\label{sec:exp-compression}

We vary the Top-$k$ budget $\rho$ at fixed $\delta=2$. Realized support tracks
the requested budget, payload grows approximately linearly, and detection
saturates after a low-rate knee (Figure~\ref{fig:compression-sweeps}).
\arxivonly{Full numeric results are reported in Appendix
Tables~\ref{tab:compression-noisy-app} and~\ref{tab:compression-perfect-app}.}

The Noisy sweep reaches the primary operating point at $\rho=0.3$, where the payload is 0.5131 MiB/frame and AP@0.5 is 0.8445. Relative to the full-spatial codec endpoint $\rho=1$, this operating point reduces payload by 69.9\% while losing 0.84 AP points at IoU 0.5. Increasing $\rho$ from 0.5 to 0.75 raises
payload by approximately 50\% but changes AP@0.5 by only 0.07 points. The Perfect sweep shows the same saturation pattern: at $\rho=0.3$, payload is reduced by 69.9\% relative to its $\rho=1$ endpoint, with losses of 0.72 AP points at IoU 0.5 and 1.88 points at IoU 0.7. These results support $\rho$ as an effective runtime rate control.

\paragraph{Fixed-payload decoder depth.}
\label{sec:exp-depth}

The depth experiment isolates receiver-side reconstruction computation by holding $\rho=0.3$ and the transmitted payload fixed. The checkpoint is trained with $\delta=2$; at evaluation, the same shared-weight DIM is forced to execute exactly $\delta$ iterations. Changing $\delta$ therefore changes the effective receptive field but neither the payload nor the number of learned parameters.


\begin{table}[htbp]
\centering
\small
\setlength{\tabcolsep}{2.5pt}
\begin{tabular}{@{}lccrrr@{}}
\toprule
Setting & $\sigma_{xyz}$ & Payload$\downarrow$ & AP@0.5$\uparrow$ &
AP@0.7$\uparrow$ & \shortstack{Miss.\\S-$L_1$}$\downarrow$ \\
\midrule
\multicolumn{6}{@{}l}{\emph{Decoder depth} (Noisy protocol, $\rho=0.3$)} \\
\cmidrule(lr){1-6}
$\delta=0$ & 0.2 & 0.5131 & 0.8344 & 0.6451 & 0.002193 \\
$\delta=1$ & 0.2 & 0.5131 & 0.8382 & 0.6440 & 0.002035 \\
$\delta=2$ & 0.2 & 0.5131 & \textbf{0.8445} & \textbf{0.6465} & \textbf{0.001979} \\
$\delta=3$ & 0.2 & 0.5131 & 0.8399 & 0.6399 & 0.002356 \\
$\delta=4$ & 0.2 & 0.5131 & 0.8187 & 0.6213 & 0.003886 \\
$\delta=5$ & 0.2 & 0.5131 & 0.7973 & 0.6119 & 0.008181 \\
\midrule
\multicolumn{6}{@{}l}{\emph{Localization noise} ($\rho=0.3$, $\delta=2$)} \\
\cmidrule(lr){1-6}
Delay-only & 0.0 & 0.5131 & 0.8596 & 0.6421 & -- \\
Noisy & 0.2 & 0.5131 & 0.8445 & 0.6465 & -- \\
High noise & 0.5 & 0.5131 & 0.8044 & 0.5719 & -- \\
\bottomrule
\end{tabular}
\caption{Fixed-budget decoder-depth and localization-noise diagnostics. The
upper block varies $\delta$ at $\rho=0.3$ and $\sigma_{xyz}=0.2$ m and Miss. S-$L_1$ denotes the missing-feature reconstruction
proxy. The lower block varies $\sigma_{xyz}$ (m) at $\rho=0.3$ and $\delta=2$}
\label{tab:depth-noise}
\end{table}

Two DIM iterations give the best observed operating point, improving AP@0.5 by
1.01 points over $\delta=0$. The theorem concerns an optimal decoder allowed
to use at most $\delta$ steps, whereas the deployed DIM is trained at
$\delta=2$ and forced to execute each selected depth. Because its residual
updates are neither idempotent nor contractive, longer unrolling can accumulate
error: from $\delta=2$ to 5, missing-feature Smooth-$L_1$ rises $4.13\times$
and AP@0.5 falls 4.72 points. Choosing the best iterate $t\leq\delta$ instead
gives a non-decreasing AP@0.5 envelope that plateaus at $\delta=2$, consistent
with the abstract at-most-depth feasible set. The fixed-payload sweep remains an
implementation diagnostic, not verification of a lower required rate.

\paragraph{Localization-noise sensitivity.}
\label{sec:exp-robustness}

At fixed $\rho=0.3$, $\delta=2$, payload, and 100-ms offset, the lower block of
Table~\ref{tab:depth-noise} varies translation noise. \emph{Delay-only} disables
pose perturbations but remains asynchronous; \emph{Noisy} uses
$\sigma_{xyz}=0.2$ m and $\sigma_{\mathrm{yaw}}=0.2^\circ$; \emph{High noise}
raises only $\sigma_{xyz}$ to $0.5$ m. This increase reduces AP@0.5/AP@0.7 by
4.01/7.46 points, with the larger AP@0.7 loss indicating greater sensitivity
of precise box alignment. This single-method stress test exposes sensitivity
to stronger pose error; it is not a comparative robustness claim or a universal
failure threshold.

\section{Conclusion}
\label{sec:conc}

We operationalize a closure-fidelity perspective on ego-conditioned V2X
communication in SemRD-V2X through exact-budget support, compact coding, and
bounded reconstruction. The method is competitive on V2XSet and DAIR-V2X; in
the controlled V100 comparison, it improves AP@0.5/AP@0.7 by 4.13/8.57 points
and reduces analytical payload by 96.2\% (26.6$\times$), with 3.81\% higher
mean compute latency. This establishes a practical communication--accuracy--
computation trade-off while keeping analytical payload and neural
reconstruction distinct from deductive rate and closure distortion. Future
work will combine sparse execution with packet-level wireless measurements and
extend evaluation to multimodal and temporal cooperative perception.



\bibliography{paper}

@misc{xu2026rate,
  author       = {Xu, Jianfeng},
  title        = {Rate-Distortion Theory for Deductive Sources under Closure Fidelity},
  year         = {2026},
  eprint       = {2604.15698},
  archivePrefix = {arXiv},
  primaryClass = {cs.IT},
  url          = {https://arxiv.org/abs/2604.15698}
}

@article{xu2024v2x,
  author  = {Xu, Runsheng and Chen, Chia-Ju and Tu, Zhengzhong and Yang, Ming-Hsuan},
  title   = {V2X-ViTv2: Improved Vision Transformers for Vehicle-to-Everything Cooperative Perception},
  journal = {IEEE Transactions on Pattern Analysis and Machine Intelligence},
  volume  = {47},
  number  = {1},
  pages   = {650--663},
  year    = {2025},
  month   = jan
}

@inproceedings{v2xvit2022,
  author    = {Xu, Runsheng and Xiang, Hao and Tu, Zhengzhong and Xia, Xin and Yang, Ming-Hsuan and Ma, Jiaqi},
  title     = {{V2X-ViT}: Vehicle-to-Everything Cooperative Perception with Vision Transformer},
  booktitle = {European Conference on Computer Vision (ECCV)},
  pages     = {107--124},
  year      = {2022}
}

@inproceedings{opv2v2021,
  author    = {Xu, Runsheng and Xiang, Hao and Han, Xin and Liu, Xiaqing and Huang, Tianwei and Qin, Hao and Liu, Mingyu and Ma, Jiaqi},
  title     = {{OPV2V}: An Open Benchmark Dataset and Fusion Pipeline for Perception with Vehicle-to-Vehicle Communication},
  booktitle = {IEEE International Conference on Robotics and Automation (ICRA)},
  pages     = {2583--2589},
  year      = {2021}
}

@inproceedings{dairv2x2022,
  title={Dair-v2x: A large-scale dataset for vehicle-infrastructure cooperative 3d object detection},
  author={Yu, Haibao and Luo, Yizhen and Shu, Mao and Huo, Yiyi and Yang, Zebang and Shi, Yifeng and Guo, Zhenglong and Li, Hanyu and Hu, Xing and Yuan, Jirui and Nie, Zaiqing},
  booktitle={Proceedings of the IEEE/CVF Conference on Computer Vision and Pattern Recognition},
  pages={21361--21370},
  year={2022}
}

@inproceedings{fcooper2019,
  author    = {Chen, Qi and Ma, Xiao and Tang, Sihai and Guo, Jingda and Yang, Qiaoqiao and Fu, Song},
  title     = {{F-Cooper}: Feature-based Cooperative Perception for Autonomous Vehicle Multi-Agent Systems},
  booktitle = {International Conference on Autonomous Agents and Multi-Agent Systems (AAMAS)},
  year      = {2019}
}

@inproceedings{v2vnet2020,
  author    = {Wang, Tsun-Hsuan and Manivasagam, Sivabalan and Liang, Manmohan and Yang, Bin and Zeng, Wei and Urtasun, Raquel},
  title     = {{V2VNet}: Vehicle-to-Vehicle Communication for Joint Perception and Prediction},
  booktitle = {European Conference on Computer Vision (ECCV)},
  pages     = {605--621},
  year      = {2020}
}

@article{disconet2021,
  title={Learning distilled collaboration graph for multi-agent perception},
  author={Li, Yiming and Ren, Shunli and Wu, Pengxiang and Chen, Siheng and Feng, Chen and Zhang, Wenjun},
  journal={Advances in Neural Information Processing Systems},
  volume={34},
  pages={29541--29552},
  year={2021}
}

@inproceedings{who2com2020,
  author    = {Liu, Yen-Cheng and Tian, Junjiao and Ma, Chih-Yao and Glaser, Nathan and Kuo, Chia-Wen and Kira, Zsolt},
  title     = {{Who2com}: Collaborative Perception via Learnable Handshake Communication},
  booktitle = {2020 IEEE International Conference on Robotics and Automation (ICRA)},
  pages     = {6876--6883},
  year      = {2020},
  doi       = {10.1109/ICRA40945.2020.9197364}
}

@inproceedings{when2com2020,
  author    = {Liu, Yen-Cheng and Tian, Junjiao and Glaser, Nathaniel and Kira, Zsolt},
  title     = {{When2com}: Multi-Agent Perception via Communication Graph Grouping},
  booktitle = {Proceedings of the IEEE/CVF Conference on Computer Vision and Pattern Recognition (CVPR)},
  pages     = {4105--4114},
  year      = {2020},
  doi       = {10.1109/CVPR42600.2020.00416}
}

@inproceedings{where2comm2022,
  author    = {Hu, Yue and Fang, Shuai and Lei, Zixing and Zhong, Yiqi and Chen, Siheng},
  title     = {{Where2Comm}: Communication-Efficient Collaborative Perception with Spatial Questioner},
  booktitle = {Neural Information Processing Systems (NeurIPS)},
  year      = {2022}
}

@inproceedings{how2comm2023,
  author    = {Lei, Zixing and Hu, Yue and Yang, Yi and Zhong, Yiqi and Chen, Siheng},
  title     = {{How2Comm}: Communication-Efficient and Collaboration-Pragmatic Multi-Agent Perception},
  booktitle = {ACM Multimedia (ACM MM)},
  year      = {2023}
}

@inproceedings{cobevt2022,
  title     = {{CoBEVT}: Cooperative Bird's Eye View Semantic Segmentation with Sparse Transformers},
  author    = {Xu, Runsheng and Tu, Zhengzhong and Xiang, Hao and Shao, Wei and Zhou, Bolei and Ma, Jiaqi},
  booktitle = {Proceedings of the 6th Conference on Robot Learning},
  year      = {2022}
}

@inproceedings{latencyaware2022,
  author    = {Lei, Zixing and Ren, Shunli and Hu, Yue and Zhang, Wenjun and Chen, Siheng},
  title     = {Latency-Aware Collaborative Perception},
  booktitle = {Computer Vision -- ECCV 2022},
  pages     = {316--332},
  year      = {2022},
  doi       = {10.1007/978-3-031-19824-3_19}
}

@inproceedings{heal2024,
  author    = {Lu, Yifan and Hu, Yue and Zhong, Yiqi and Wang, Dequan and Wang, Yanfeng and Chen, Siheng},
  title     = {An Extensible Framework for Open Heterogeneous Collaborative Perception},
  booktitle = {International Conference on Learning Representations (ICLR)},
  year      = {2024},
  url       = {https://openreview.net/forum?id=KkrDUGIASk}
}

@INPROCEEDINGS{edgev2x2021,
  author={Yu, Ruozhou and Yang, Dejun and Zhang, Hao},
  booktitle={2021 IEEE/ACM Symposium on Edge Computing (SEC)}, 
  title={Edge-Assisted Collaborative Perception in Autonomous Driving: A Reflection on Communication Design}, 
  year={2021},
  volume={},
  number={},
  pages={371-375},
  doi={10.1145/3453142.3491413}}

@inproceedings{robcoop2023,
  title={A cooperative perception system robust to localization errors},
  author={Song, Zhiying and Wen, Fuxi and Zhang, Hailiang and Li, Jun},
  booktitle={2023 IEEE Intelligent Vehicles Symposium (IV)},
  pages={1--6},
  year={2023},
  organization={IEEE}
}

@article{chiu2023selective,
  title   = {Selective Communication for Cooperative Perception in End-to-End Autonomous Driving},
  author  = {Chiu, Hsu-kuang and Smith, Stephen F.},
  journal = {arXiv preprint arXiv:2305.17181},
  year    = {2023},
  eprint  = {2305.17181},
  archivePrefix = {arXiv},
  primaryClass = {cs.RO}
}

@inproceedings{vaswani2017,
  author    = {Vaswani, Ashish and Shazeer, Noam and Parmar, Niki and Uszkoreit, Jakob and Jones, Llion and Gomez, Aidan N. and Kaiser, Lukasz and Polosukhin, Illia},
  title     = {Attention Is All You Need},
  booktitle = {Neural Information Processing Systems (NeurIPS)},
  year      = {2017}
}

@article{shannon1948,
  author  = {Shannon, Claude E.},
  title   = {A Mathematical Theory of Communication},
  journal = {Bell System Technical Journal},
  volume  = {27},
  number  = {3},
  pages   = {379--423},
  year    = {1948}
}

@book{cover2006,
  author    = {Cover, Thomas M. and Thomas, Joy A.},
  title     = {Elements of Information Theory},
  edition   = {2nd},
  publisher = {John Wiley \& Sons},
  year      = {2006}
}

@article{blahut1972,
  author  = {Blahut, Richard E.},
  title   = {Computation of Channel Capacity and Rate-Distortion Functions},
  journal = {IEEE Transactions on Information Theory},
  volume  = {18},
  number  = {4},
  pages   = {460--473},
  year    = {1972}
}

@article{arimoto1972,
  author  = {Arimoto, Suguru},
  title   = {An Algorithm for Computing the Capacity of Arbitrary Discrete Memoryless Channels},
  journal = {IEEE Transactions on Information Theory},
  volume  = {18},
  number  = {1},
  pages   = {14--20},
  year    = {1972}
}

@article{bourtsoulatze2019,
  author  = {Bourtsoulatze, Evgeny and Kurka, David Burth and G{\"u}nd{\"u}z, Deniz},
  title   = {Deep Joint Source-Channel Coding for Wireless Image Transmission},
  journal = {IEEE Transactions on Cognitive Communications and Networking},
  volume  = {5},
  number  = {3},
  pages   = {567--579},
  year    = {2019}
}

@article{xie2021deepsc,
  author  = {Xie, Huiqiang and Qin, Zhijin and Li, Geoffrey Ye and Juang, Biing-Hwang},
  title   = {Deep Learning Enabled Semantic Communication Systems},
  journal = {IEEE Transactions on Signal Processing},
  volume  = {69},
  pages   = {2663--2675},
  year    = {2021},
  doi     = {10.1109/TSP.2021.3071210}
}

@article{tishby1999,
  author  = {Tishby, Naftali and Pereira, Fernando C. and Bialek, William},
  title   = {The Information Bottleneck Method},
  journal = {Proceedings of the 37th Annual Allerton Conference on Communication, Control, and Computing},
  pages   = {368--377},
  year    = {1999}
}

@inproceedings{blau2019,
  title={Rethinking lossy compression: The rate-distortion-perception tradeoff},
  author={Blau, Yochai and Michaeli, Tomer},
  booktitle={International Conference on Machine Learning},
  pages={675--685},
  year={2019},
  organization={PMLR}
}

@article{slepian1973noiseless,
  author  = {Slepian, David and Wolf, Jack K.},
  title   = {Noiseless Coding of Correlated Information Sources},
  journal = {IEEE Transactions on Information Theory},
  volume  = {19},
  number  = {4},
  pages   = {471--480},
  year    = {1973},
  doi     = {10.1109/TIT.1973.1055037}
}

@article{wyner1976rate,
  author  = {Wyner, Aaron D. and Ziv, Jacob},
  title   = {The Rate-Distortion Function for Source Coding with Side Information at the Decoder},
  journal = {IEEE Transactions on Information Theory},
  volume  = {22},
  number  = {1},
  pages   = {1--10},
  year    = {1976},
  doi     = {10.1109/TIT.1976.1055508}
}

@inproceedings{codefilling2024,
  author    = {Hu, Yue and Peng, Juntong and Liu, Sifei and Ge, Junhao and Liu, Si and Chen, Siheng},
  title     = {Communication-Efficient Collaborative Perception via Information Filling with Codebook},
  booktitle = {Proceedings of the IEEE/CVF Conference on Computer Vision and Pattern Recognition (CVPR)},
  pages     = {15481--15490},
  year      = {2024},
  url       = {https://openaccess.thecvf.com/content/CVPR2024/html/Hu_Communication-Efficient_Collaborative_Perception_via_Information_Filling_with_Codebook_CVPR_2024_paper.html}
}

@inproceedings{pointpillars2019,
  title={Pointpillars: Fast encoders for object detection from point clouds},
  author={Lang, Alex H and Vora, Sourabh and Caesar, Holger and Zhou, Lubing and Yang, Jiong and Beijbom, Oscar},
  booktitle={Proceedings of the IEEE/CVF conference on computer vision and pattern recognition},
  pages={12697--12705},
  year={2019}
}

@inproceedings{10625456,
  author    = {An, Byoungman and Jang, Seonghyun and Jang, Soohyun and Jang, Junhyek and Shin, Daekyo and Yoon, Sanghun},
  title     = {Experimental 5G-NR-V2X Evaluation in a Real-Life Highway and Proving Ground Environment},
  booktitle = {2024 Fifteenth International Conference on Ubiquitous and Future Networks (ICUFN)},
  pages     = {13--18},
  year      = {2024},
  doi       = {10.1109/ICUFN61752.2024.10625456}
}

\end{document}